%% file: main.tex
\documentclass[letter, 10 pt, journal, twoside]{IEEEtran} 

\usepackage[english]{babel}
\usepackage[utf8]{inputenc}
\usepackage{times} 
\usepackage{cite}

\usepackage{booktabs}

\usepackage[hidelinks]{hyperref}
\usepackage[nolist]{acronym}

\usepackage{multirow}

\usepackage{subcaption}

\usepackage{siunitx}
\usepackage{graphicx} 
\usepackage{epsfig} 
\usepackage{pgfplots}
\pgfplotsset{compat=1.18}
\usepgfplotslibrary{statistics}
\usepgfplotslibrary{groupplots}

\usepackage{float}

\usepackage[export]{adjustbox}

\usepackage{xcolor}
\definecolor{myYellow}{rgb}{0.93,0.69,0.13}
\definecolor{myPurple}{rgb}{0.49,0.18,0.56}
\definecolor{myGreen}{rgb}{0.26 0.72 0.54}
\definecolor{darkgreen}{rgb}{0.272, 0.50, 0.376}
\definecolor{lightgreen}{rgb}{0.585, 0.82, 0.647}

\definecolor{cOrange}{RGB}{255,127,0}
\definecolor{cRed}{RGB}{214,39,40}
\definecolor{cBlue}{RGB}{31,119,180}
\definecolor{cGrey}{RGB}{127,127,127}
\definecolor{cGreen}{RGB}{44,160,44}

\colorlet{mydarkblue}{blue!30!black}

\usepackage{mathtools}
\usepackage{nicefrac}
\usepackage{amsmath}
\usepackage{amssymb}
\usepackage{bm} 
\usepackage{scalerel} 

\usepackage{amsthm}
\newtheoremstyle{compact}
  {4pt}   
  {4pt}   
  {\itshape}
  {}
  {\bfseries}
  {.}
  {0.5em}
  {}

\theoremstyle{compact}
\newtheorem{theorem}{Theorem}
\newtheorem{lemma}[theorem]{Lemma}
\newtheorem{proposition}[theorem]{Proposition}

\newtheorem{definition}{Definition}
\newtheorem{assumption}{Assumption}
\newtheorem{remark}[theorem]{Remark}

\newcommand{\diag}{\operatorname{diag}}

\DeclareMathAlphabet{\pazocal}{OMS}{zplm}{m}{n}

\usepackage{etoolbox}
\makeatletter%
\AfterPreamble{%
	\usepackage{hyperref}%
	\let\oldhypertarget\hypertarget%
	\renewcommand{\hypertarget}[2]{%
		\oldhypertarget{#1}{#2}%
		\protected@write\@mainaux{}{%
			\string\expandafter\string\gdef%
			\string\csname\string\detokenize{#1}\string\endcsname{#2}%
		}%
	}%
	\newcommand{\myhyperlink}[1]{%
		\hyperlink{#1}{\csname #1\endcsname}%
	}%
}
\makeatother%

\usepackage[ruled,vlined,linesnumbered]{algorithm2e}

\makeatletter
\def\BState{\State\hskip-\ALG@thistlm}
\makeatother

\usepackage{tikz}
\pgfdeclarelayer{foreground}
\pgfsetlayers{background, main, foreground}
\usetikzlibrary{quotes, angles, backgrounds, arrows, automata, shapes, positioning, calc, through, spy, decorations.pathreplacing, decorations.markings, arrows.meta, automata, petri, shapes.multipart, patterns}

\newlength{\panelwd}
\newlength{\panelht}

\tikzset{
    imglabel/.style={
      rectangle,
      inner sep=2pt,
      text=black,
      minimum height=1em,
      text centered,
      fill=white,
      fill opacity=1.0,
      text opacity=1,
      anchor=south west,
    },
  }
\tikzset{
	state/.style={
		rectangle,
		draw=black, very thick,
		minimum height=1.0em,
		text centered,
	},
}
\tikzset{
  on each segment/.style={
    decorate,
    decoration={
      show path construction,
      moveto code={},
      lineto code={
        \path [#1]
        (\tikzinputsegmentfirst) -- (\tikzinputsegmentlast);
      },
      curveto code={
        \path [#1] (\tikzinputsegmentfirst)
        .. controls
        (\tikzinputsegmentsupporta) and (\tikzinputsegmentsupportb)
        ..
        (\tikzinputsegmentlast);
      },
      closepath code={
        \path [#1]
        (\tikzinputsegmentfirst) -- (\tikzinputsegmentlast);
      },
    },
  },
  mid arrow/.style={postaction={decorate,decoration={
        markings,
        mark=at position .5 with {\arrow[#1]{stealth}}
      }}},
}

\tikzset{
  half circle/.style={
      semicircle,
      shape border rotate=180,
      anchor=chord center,
      minimum size=5mm
      }
}

\graphicspath{{./figures/}}

\newcommand{\coderepo}{{\url{https://anonymous.4open.science/r/tilt-certified-readiness-418A}}}

\begin{document}

\title{\LARGE \bf Tilt as a Certified Resource: Preserving\\ Motor Wrench-Rate Authority on Articulated Multirotors}

\author{Giuseppe Silano$^{1}$ and Martin Saska$^{2}$ 
    \thanks{$^1$Giuseppe Silano is with the Department of Power Generation Technologies and Materials, Ricerca sul Sistema Energetico (RSE) S.p.A., 20134 Milan, Italy, and also with the Department of Cybernetics, Czech Technical University in Prague, 12135 Prague, Czech Republic (e-mail: {\tt\small giuseppe.silano@fel.cvut.cz}).}
    \thanks{$^2$Martin Saska is with the Department of Cybernetics, Czech Technical University, 12135 Prague, Czech Republic (e-mail: {\tt\small martin.saska@fel.cvut.cz}).} 
    \thanks{This work was partially funded by the research fund for the Italian Electrical System (decree n. 388, Nov.~6th, 2024), by the GAČR project no.~26-22419S, by the CTU grant no. SGS26/077/OHK3/1T/13, and by the European Union under ROBOPROX reg. no. CZ.02.01.01/00/22\_008/0004590.}
    \thanks{The author used Google Gemini 3 and Anthropic Claude 5 Opus for proofreading and exposition refinement, but independently verified all outputs and assumes full responsibility for the article's content.}
    %
}

\maketitle
\thispagestyle{empty} 
\pagestyle{empty} 


\begin{acronym}
    \acro{BEM}[BEM]{Blade-Element-Momentum}
    \acro{CBF}[CBF]{Control Barrier Function}
    \acro{DAAM}[DAAM]{Drag-Aware Aerodynamic Manipulability}
    \acro{ESC}[ESC]{Electronic Speed Controller}
    \acro{MRAV}[MRAV]{Multi Rotor Aerial Vehicle}
    \acro{MPC}[MPC]{Model Predictive Control}
    \acro{NN}[NN]{Neural Network}
    \acro{QP}[QP]{Quadratic Program}
    \acro{RMS}[RMS]{Root Mean Square}
    \acro{UPC-QP}[UPC-QP]{Unified Physical-Command Quadratic Program}
\end{acronym}



\begin{abstract}
    Fully-actuated multirotor aerial vehicles must not only track nominal wrenches but retain the ``readiness'' to modulate them rapidly under disturbances. Classical effort-minimizing allocators ignore this dynamic limit, whereas maximizing readiness leads to topologically disconnected optimal sheets demanding physically impossible actuator rates. Enforcing a readiness safety floor on fixed-geometry symmetric platforms further encounters a zero-sum degeneracy: motor-speed redistribution cannot improve authority without conceding wrench tracking.
    This paper uses active morphology to break the degeneracy, treating servo tilt as a geometric resource supplying authority-recovery directions unavailable to static rotors. We construct a configuration-dependent, motor-only readiness certificate---the log-volume of the reachable wrench-rate set---that explicitly excludes servo capacity, preventing a ``ghost capacity fallacy'' in which the certificate would falsely credit slow mechanical kinematic limits instead of collapsing accurately at motor saturation. The certificate is enforced as a Control Barrier Function (CBF) within a Unified Physical-Command Quadratic Program acting on motor torques and servo setpoints. Closed-loop simulations of an articulated octorotor under severe gust disturbances show classical allocators diverging and uncertified articulated allocators violating the safety floor, while the proposed CBF filter bounds the system state and preserves vehicle authority.
\end{abstract}






\section{Introduction}
\label{sec:introduction}

Fully-actuated \acfp{MRAV} have expanded aerial robotics' operational envelope \cite{Hamandi2021, Rashad2020}, decoupling translational and rotational dynamics through overactuation to track arbitrary six-degree-of-freedom trajectories and sustain physical interaction with the environment \cite{Ollero2022}. This redundancy introduces the classically ill-posed control allocation problem: mapping a $6$-dimensional target wrench to a higher-dimensional actuator-command space \cite{Johansen2013}.

The classical answer minimizes actuator effort via the Moore--Penrose pseudo-inverse or a quadratic program \cite{Johansen2013}---computationally light, but blind to a distinct dynamic limit we call \emph{wrench-rate readiness}: the rotors' remaining capacity to generate rapid wrench variations. Propeller thrust and drag both scale quadratically with spin rate, imposing competing limits: a slow rotor has a shallow thrust slope and cannot modulate force quickly, while a rotor near saturation has spent its torque budget fighting drag and has no headroom to accelerate. Readiness therefore peaks at an interior \emph{sweet spot} and collapses at both extremes of the speed range.

Consider an articulated \ac{MRAV} holding a contact force against a wall when a gust arrives: it may track its commanded wrench exactly, every rotor inside its speed limits, effort-minimization problem solved to optimality, and still fail---the rotors spinning fast enough that drag has consumed the torque budget, leaving no capacity to accelerate \cite{Bicego2020, Franchi2026aeropromptness}. Producing a nominal wrench and retaining the authority to change it are distinct capabilities.

Recent geometric work formalizes the latter through a drag-aware authority co-metric, \acf{DAAM}, whose log-determinant measures the reachable wrench-rate volume \cite{Franchi2026aeropromptness, FranchiVADA}. 
The co-metric $D(v)$, with $v \in \mathbb{R}^n$ the rotor spin rates, maps actuator bounds into task space; and the \emph{readiness level} $L(v)=\ln\det D(v)$ is that set's scalar log-volume. Maximizing $L$ along the \emph{task fiber}---the continuum of redundant rotor speeds satisfying the current wrench command---repels the allocation from both collapse modes; two obstacles remain.

\emph{The first is topological:} the fiberwise maximizer is set-valued, with optimal selections on disconnected sheets. The signed-quadratic thrust map partitions rotor-speed space into \emph{sign orthants}---convex regions where every rotor keeps a fixed spin direction, bounded by loci where it reverses. At such a zero crossing the thrust slope vanishes, the allocation Jacobian degenerates, and the actuator rate needed to sustain a bounded wrench rate diverges, so the fiber decomposes into components a bounded-rate allocator cannot cross (Figure~\ref{fig:section}). Low-pass filtering bounds the rate demand but introduces phase lag, dragging the trajectory through the low-authority crossing rather than avoiding it.
 
\emph{The second is geometric}, and motivates this paper. Rather than maximize authority, one may certify it via a floor on $L$ as a \acf{CBF} \cite{Kurtz2021, Zhang2024}. On symmetric fixed-geometry vehicles this meets a structural limitation \cite{Silano2026Readiness}: at a regular constrained optimum---interior in rotor speed, wrench-rate Jacobian of full row rank---every motor-speed perturbation preserving the commanded wrench leaves $L$ unchanged to first order, so the constant-wrench ascent available to motor redistribution is zero. Physically, the readiness gained by accelerating one rotor is cancelled by decelerating another; authority at the floor can then be recovered only by conceding wrench tracking. Section \ref{sec:art} formalizes this; the configuration that efficient hover selects is exactly where motor-only certification becomes indefensible.

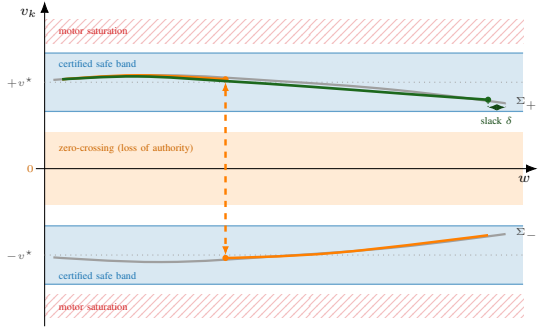
\begin{figure}[tb]
    \centering
    \resizebox{0.825\columnwidth}{!}{%
        \input{figures/tikz/fig_section.tex}}
    \vspace{-0.3em}    
    \caption{\small Allocation section over commanded wrench $w$ for rotor speed $v_k$. Greedy maximization jumps between optimal sheets $\Sigma_\pm$ (dashed), demanding physically impossible actuator rates. The proposed barrier ($h \ge 0$) strictly confines the allocation to a single certified band, resolving authority-tracking conflicts via an explicit wrench slack $\delta$.}
    \label{fig:section}
    \vspace{-1em}
\end{figure}

We break the degeneracy through active morphology \cite{Rashad2020, Hamandi2021}, treating servo tilt as geometric resource rather than a tracking variable: physically tilting the rotors alters the allocation matrix's geometry, contributing wrench-rate directions unreachable by redistributing spin among fixed rotors. We are deliberate about scope---articulation is bidirectional, and an arbitrary tilt may raise or lower the metric---but a strictly favorable direction always \emph{exists} where static motor redistribution has none, and a certified allocator will select it.

We therefore certify rather than maximize: a configuration-dependent, motor-only readiness certificate enforces a floor on the log-determinant, crediting no servo authority and remaining conservative in the L\"owner order---not merely from caution, but because a servo-inclusive metric stays finite at motor saturation and would report a vehicle with no torque headroom as healthy. Restricting the metric to motor capacity preserves the divergence of $L$ at saturation and, with it, the property that a single inequality fences both collapse modes (Section \ref{sec:art}).

\emph{Contributions.} The core contributions of this manuscript are fourfold:
\begin{itemize}\itemsep1pt \parskip0pt \topsep2pt
    \item A first-order geometric characterization of the motor-only readiness degeneracy at a regular constrained optimum, with explicit regularity hypotheses (Section~\ref{sec:art}).
    \item A configuration-dependent readiness certificate, conservative in the L\"owner order, whose conservatism is what preserves the saturation indication that a servo-inclusive metric destroys (Sections~\ref{sec:model}--\ref{sec:art}).
    \item Enforcement of that certificate on motor and servo commands through a memoryless \ac{UPC-QP} with relative-degree-one barrier dynamics and an online feasibility margin, structurally free of integrator wind-up (Section~\ref{sec:upcqp}).
    \item A controlled three-way evaluation isolating the causal contribution of the readiness constraint from articulation alone under severe disturbance, in which an articulated allocator without the certificate still violates the floor (Section~\ref{sec:sim}).
\end{itemize}



\section{Related Work}
\label{sec:relatedWorkd}

Allocation for overactuated systems is mature \cite{Johansen2013}; \ac{MRAV} specialization is well developed: tilted-rotor platforms decouple translational and rotational dynamics \cite{Rashad2020, Allenspach2020}; actuation properties are characterized via force--moment decoupling \cite{Bicego2020}; surveys cover the platform class \cite{Hamandi2021} and applications \cite{Ollero2022}. The dominant allocator---pseudo-inverse or effort-minimizing program---ignores actuator bounds and saturates afterwards. Constrained dynamic allocation embeds those limits directly \cite{Harkegard2004}, reference governors enforce admissibility \cite{Convens2017}, full-pose tracking extends to limited actuation \cite{Hamandi2023} and bounded lateral force \cite{Franchi2018}, and null-space methods exploit redundancy instantaneously \cite{Su2021} or over a horizon \cite{Pretto2026}. All certify that the commanded wrench is currently \emph{attainable}; none floors the capacity to \emph{change} it, which is what decides whether a gust can be answered.

Measuring capability by reachable-set volume originates with Yoshikawa's manipulability index \cite{Yoshikawa1985}. \ac{DAAM} \cite{Franchi2026aeropromptness} adapts this to \acp{MRAV} by weighting each rotor by its \emph{aerodynamic} state, so the metric collapses as drag consumes torque margin, trading off against energy when maximized along the task fiber \cite{FranchiVADA}. We adopt this co-metric but not its use: a floor rather than a maximum sidesteps the set-valued, disconnected fiberwise maximizer.

Such a floor is naturally a \ac{CBF} \cite{Ames2019}. Determinant-type capability barriers appear in manipulator control: \cite{Kurtz2021} certify singularity avoidance with an exponential \ac{CBF} on the manipulability index. 
Closest to this work, \cite{Forghani2026singularity} certify a margin from rank-deficient regions of a state-dependent input--output map, using one barrier per eigenvalue---a construction requiring \emph{simple} eigenvalues. This is where our setting diverges: a symmetric airframe carries exactly repeated eigenvalues by construction, so per-eigenvalue barriers lose differentiability precisely where we need them, whereas $\ln\det D$ stays smooth wherever $D\succ0$. Their map is also weighted uniformly, while ours carries the drag-dependent capacity producing the collapse in the first place.

Feasibility of the enforcing program is a known concern, addressed through conditions on the optimal control problem, input-constrained constructions, soft-minimum barriers, and composite certificates for actuator lag \cite{Ames2019}; we inherit rather than resolve it, since our barrier acts on physical commands whose admissible set is a box, giving a closed-form per-step margin that we monitor instead of assume.

Rotor tilting is itself established for reshaping the feasible wrench set \cite{Rashad2020, Allenspach2020}, with recent omnidirectional platforms exploring actuation-aware allocation experimentally \cite{Hamandi2025prototype, Pretto2026}: there, tilt \emph{reaches} wrenches; we use it to \emph{preserve} the ability to change them. Fault-tolerant control is the complement in time, accommodating rotor loss after it occurs \cite{Mueller2014}; readiness acts before.




\section{System Model and Readiness Certificate}
\label{sec:model}

This section fixes the vehicle model, constructs the readiness co-metric, and separates three objects the paper keeps distinct. Every quantity here is computable from the current state by a single $m\times m$ factorization.



\subsection{Physical actuator dynamics}
\label{sec:actuators}

The vehicle carries $n>m=6$ actively articulable rotor groups. Rotor $i$ spins at a signed rate $v_i$, axis tilted by servo angle $\alpha_i$. Physical inputs are motor torques $\tau\in\mathbb{R}^n$ and the servo setpoints $\alpha_c\in\mathbb{R}^n$:
\begin{align}
    J_{m,i}\,\dot v_i &= \tau_i - b_i v_i|v_i|, & |\tau_i| &\le \bar\tau_i,
    \label{eq:motor}\\
    \dot\alpha_i &= \tfrac{1}{\tau_s}\big(\alpha_{c,i}-\alpha_i\big), &
    |\dot\alpha_i| &\le \bar u, \ \ \alpha_i\in[\underline\alpha, \bar\alpha],
    \label{eq:servo}
\end{align}
with lumped motor--propeller inertia $J_{m,i}$, drag coefficient $b_i$, torque limit $\bar\tau_i$, servo time constant $\tau_s$, and servo rate limit $\bar u$. The first-order rotor model with quadratic drag \eqref{eq:motor} is standard and experimentally identified \cite{Bicego2020}; the servo is a symmetric first-order lag, chosen to keep physical commands in the barrier's first derivative---whether that suffices to avoid high-order constructions \cite{Breeden2021} is established in Section~\ref{sec:upcqp}, not assumed.

Equation \eqref{eq:motor} determines the certificate's basis. For $v_i>0$ the admissible spin acceleration satisfies $|\dot v_i|\le \bar a_i(v_i)$, where
\begin{align}
    \bar a_i(v_i) = (\bar\tau_i - b_i v_i^2)/J_{m,i},
    \qquad
    v_i^{\mathrm{sat}} = \sqrt{\bar\tau_i/b_i},
    \label{eq:abar}
\end{align}
is the \emph{symmetric acceleration capacity}: maximal at rest, decreasing monotonically to zero at the drag-limited terminal speed $v_i^{\mathrm{sat}}$, where drag consumes the entire torque budget---formalizing that a rotor may be far from any speed limit yet have no capacity to accelerate.



\subsection{Wrench map and wrench-rate jacobians}

The body wrench follows the signed-quadratic thrust map
\begin{align}
    w = A(\alpha)\,\phi(v)\in\mathbb{R}^{m},
    \qquad \phi_i(v_i)=v_i|v_i|,
    \label{eq:wrench}
\end{align}
where $A_{\bullet i} \in \mathbb{R}^m$ collects the thrust direction and induced drag torque of rotor $i$ at cant $\alpha_i$. 
With $A'_{\bullet i}=\partial A_{\bullet i}/\partial\alpha_i$ the tilt direction of column $i$, differentiating \eqref{eq:wrench} gives
\begin{align}
    \dot w = J_v(v,\alpha)\,\dot v + J_\alpha(v,\alpha)\,\dot\alpha,
    \label{eq:wrenchrate}
\end{align}
with $J_v = 2A(\alpha)\diag(|v|)$ and $[J_\alpha]_{\bullet i}=A'_{\bullet i}\,\phi_i(v)$. Substituting \eqref{eq:motor}--\eqref{eq:servo} renders the wrench rate affine in the physical command $u=[\tau^{\!\top},\alpha_c^{\!\top}]^{\!\top}$,
\begin{align}
    \dot w = M(v,\alpha)\,u + d(v,\alpha),
    \label{eq:affine}
\end{align}
with input map and drift
\begin{align}
    \resizebox{0.90\columnwidth}{!}{$
    M = \big[\,J_v D_m \;\; \frac{1}{\tau_s}J_\alpha\,\big], \quad
    d = -J_v D_m\, b\odot\phi(v) - \frac{1}{\tau_s}J_\alpha\,\alpha,
    $}
    \label{eq:Md}
\end{align}
where $D_m=\diag(J_{m,i}^{-1})$, $b=(b_1,\dots,b_n)$, and $\odot$ is the elementwise product. Drift $d$ collects command-independent terms: drag torque to overcome and the servo lag on current tilt. Equation \eqref{eq:affine} is the form the Section~\ref{sec:upcqp} allocator acts on.



\subsection{The motor-only readiness co-metric}

Weighting each rotor by its capacity \eqref{eq:abar}, the wrench rates achievable by the motors at
fixed articulation contain the ellipsoid $\mathcal{E}=\{J_v W^{1/2}\zeta:\|\zeta\|_2\le1\}$ with $W=\diag(\bar a_i^2)$, whose squared volume is proportional to the determinant of the \ac{DAAM} co-metric
\begin{align}
    D(v,\alpha) \;=\; J_v W J_v^{\!\top} \;=\; 4\,A(\alpha)\,\Psi(v)\,A(\alpha)^{\!\top},
    \label{eq:D}
\end{align}
with $\Psi=\diag(\psi_i)$ and the per-rotor readiness weight is
\begin{align}
    \psi_i(v_i) = v_i^2\,\bar a_i(v_i)^2 ,
    \label{eq:psi}
\end{align}
where a prime denotes differentiation with respect to the rotor's own spin rate, $\psi_i'=\mathrm{d}\psi_i/\mathrm{d}v_i$. The scalar readiness level and the shifted safety barrier are
\begin{align}
    L(v,\alpha) = \ln\det D(v,\alpha),
    \quad
    h(v,\alpha) = L(v,\alpha) - \ell_{\min},
    \label{eq:L}
\end{align}
where $\ell_{\min}\in\mathbb{R}$ is the designer-chosen \emph{readiness floor}. 
The certified set is $\{(v,\alpha): h(v,\alpha)\ge0\}$, and Section~\ref{sec:sim} sets $\ell_{\min}$ relative to nominal hover readiness.
 
Two properties of \eqref{eq:psi} explain why a single scalar inequality suffices: $\psi_i$ vanishes at $v_i=0$ (zero thrust slope, no wrench rate contribution) and at $v_i=v_i^{\mathrm{sat}}$ ($\bar a_i=0$, no acceleration capacity)---in either case $\det D$ loses that rotor's contribution, so a floor on $L$ fences \emph{both} failure modes, dropout and saturation, without auxiliary logic. Moreover $\mathcal{E}$ is the image of the unit $2$-ball, contained in the actuator-rate box, so \eqref{eq:D} under-approximates the true reachable zonotope (Figure~\ref{fig:zonotope}) and $L$ never overstates authority; we avoid the maximum-volume (L\"owner--John) ellipsoid, since the $2$-ball image is closed-form and differentiable in $(v,\alpha)$, and the resulting conservatism is one-sided, hence safe.

\begin{figure}[tb]
    \centering
    \resizebox{0.95\columnwidth}{!}{%
        \input{figures/tikz/fig_geometry.tex}}
    \vspace{-0.5em}    
    \caption{\small Readiness certificate geometry. (a) Actuator-rate box bounded by symmetric capacity $\bar{a}_i(v_i)$. (b) Reachable wrench-rate zonotope under-approximated by inscribed ellipsoid $\mathcal{E}$. (c) Readiness level $L = \ln \det D$ measures the log-volume of this inner approximation.}
    \label{fig:zonotope}
    \vspace{-0.75em}
\end{figure}
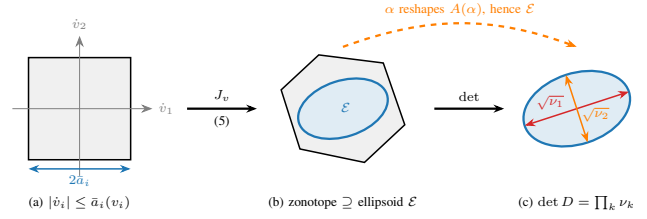

\begin{remark}[Three distinct objects]\label{rem:threeobjects}
The motor-only co-metric $D(v,\alpha)$ in \eqref{eq:D}, depending only on current state, is the \emph{sole} basis of barrier $h$; the augmented metric $D_{\mathrm{art}} = D + \sum_i \rho_i A'_{\bullet i}A'^{\top}_{\bullet i}$, with $\rho_i=\bar u^2\phi_i(v)^2$, accounting for wrench rate from servo motion, is used in Section~\ref{sec:art} only to quantify the authority the certificate declines to credit and never enters a control constraint; and the plant \eqref{eq:motor}--\eqref{eq:servo} supplies the actuator rates, with $u$ the only quantity the allocator commands. Proposition~\ref{prop:optionA} shows that restricting the certificate to motor capacity is the right choice.
\end{remark}



\subsection{Gradients in closed form}
\label{sec:grads}

Both gradients needed to enforce \eqref{eq:L} are available analytically at every control instant.
 
\begin{lemma}[Readiness gradients]\label{lem:gradients}
Let $D(v,\alpha)\succ0$ and define the \emph{rotor leverage} $s_i$ and \emph{tilt cross-leverage} $c_i$,
\begin{align}
    s_i = A_{\bullet i}^{\!\top}D^{-1}A_{\bullet i},
    \qquad
    c_i = A_{\bullet i}^{\!\top}D^{-1}A'_{\bullet i}.
    \label{eq:leverage}
\end{align}
Then
\begin{align}
    \frac{\partial L}{\partial v_i} = 4\,\psi_i'(v_i)\,s_i,
    \qquad
    \frac{\partial L}{\partial \alpha_i} = 8\,\psi_i(v_i)\,c_i .
    \label{eq:grads}
\end{align}
\end{lemma}
\begin{proof}
By Jacobi's formula, $\partial\ln\det D=\operatorname{tr}(D^{-1}\partial D)$. Only the $i$-th term of \eqref{eq:D} depends on $v_i$ or $\alpha_i$: $\partial D/\partial v_i = 4\psi_i' A_{\bullet i}A_{\bullet i}^{\!\top}$ and $\partial D/\partial \alpha_i = 4\psi_i\big(A'_{\bullet i}A_{\bullet i}^{\!\top} + A_{\bullet i}A'^{\top}_{\bullet i}\big)$. Taking traces against $D^{-1}$ gives \eqref{eq:grads}, using symmetry of $D^{-1}$. 
\end{proof}

Both expressions share one factorization of $D$, so the cost per step is a single $m\times m$ inverse and $2n$ quadratic forms---negligible for embedded flight control. 

Define the capacity-weighted \emph{leverage score} $\sigma_i = 4\psi_i s_i$. The trace identity gives
\begin{align}
    \textstyle\sum_i \sigma_i = \operatorname{tr}(D^{-1}D) = m,
    \label{eq:trace}
\end{align}
so the $m$ wrench dimensions are shared among $n$ rotors, $\sigma_i$ the fraction of readiness volume rotor $i$ supports.
 
This also quantifies the cost of losing a rotor: setting $\psi_i=0$ removes the rank-one term $4\psi_i A_{\bullet i}A_{\bullet i}^{\!\top}$ from \eqref{eq:D}, the matrix determinant lemma 
gives $\det\big(D-4\psi_iA_{\bullet i}A_{\bullet i}^{\!\top}\big)=\det(D)\,(1-\sigma_i)$, so the readiness drop is exactly
\begin{align}
    \Delta_i = -\ln\big(1-\sigma_i\big).
    \label{eq:dropout}
\end{align}
The fraction $\sigma_i\!\to\!1$ marks an essential rotor whose loss makes $D$ singular, $\Delta_i\!\to\!\infty$. On a symmetric platform at uniform spin, $\sigma_i=m/n$, and \eqref{eq:dropout} reduces to the finite \emph{dropout gap} $\Delta=\ln\!\big(n/(n-m)\big)$---a bounded, predictable degradation.
 
Equation \eqref{eq:trace} has one further consequence: when every $\sigma_i$ is equal, $\nabla_v L$ is proportional to the vector of ones, so a constant-wrench motor redistribution trades readiness between rotors at no net gain---stated precisely with regularity conditions in Section~\ref{sec:art}.



\section{Articulation as a Readiness Resource}
\label{sec:art}

We now show that rotor-speed redistribution alone cannot defend the certificate at the configuration efficient flight selects, that articulation supplies what is missing, and that the certificate must nonetheless exclude servo authority. 



\subsection{The zero-sum degeneracy}

\begin{assumption}\label{as:regular}
$v^\star$ is interior to $\{v: 0<v<v^{\mathrm{sat}}\}$, $J_v(v^\star)$ has full row rank $m$, and $v^\star$ is a regular local maximizer of $L(\cdot,\alpha)$ subject to $w(v,\alpha)=w^\star$, so a unique Lagrange multiplier $\lambda\in\mathbb{R}^m$ exists.
\end{assumption}
 
\begin{proposition}[Zero-sum authority degeneracy]\label{prop:degeneracy}
Under Assumption~\ref{as:regular}, $\nabla_v L(v^\star)=J_v(v^\star)^{\!\top}\lambda$, and the constant-wrench authority ascent available through rotor-speed redistribution,
\begin{align}
    \gamma \;=\; \max\big\{\nabla_v L^{\!\top}\xi \;:\; J_v\xi = 0,\ |\xi|\le\bar a\big\},
    \label{eq:gamma}
\end{align}
is zero. Here $\xi\in\mathbb{R}^n$ is a candidate spin-acceleration perturbation.
\end{proposition}
 
\begin{proof}
Interiority removes the box constraints; full row rank of $J_v(v^\star)$ is the linear independence constraint qualification under which $\lambda$ exists uniquely. The first-order necessary condition for the constrained maximizer is therefore $\nabla_v L(v^\star)=J_v(v^\star)^{\!\top}\lambda$. Any feasible $\xi$ in \eqref{eq:gamma} lies in $\ker J_v(v^\star)$, whence
\begin{align}
    \nabla_v L^{\!\top}\xi
    \;=\; \big(J_v^{\!\top}\lambda\big)^{\!\top}\xi
    \;=\; \lambda^{\!\top}\big(J_v\xi\big) \;=\; 0 .
    \label{eq:zerosum}
\end{align}
The linear objective therefore vanishes on the entire feasible set, which is nonempty since $\xi=0$ is admissible, so its maximum is zero.
\end{proof}

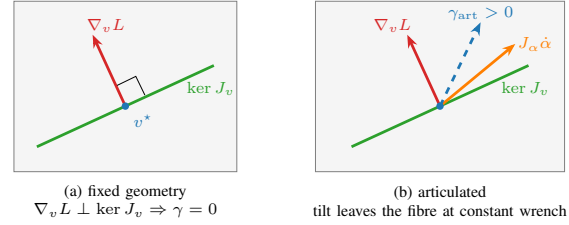
\begin{figure}[tb]
    \centering
    \resizebox{0.85\columnwidth}{!}{%
        \input{figures/tikz/fig_degeneracy.tex}}
    \vspace{-0.5em}    
    \caption{\small Authority ascent geometry. (a) At $v^\star$, the readiness gradient is orthogonal to $\ker J_v$, yielding zero constant-wrench ascent ($\gamma=0$). (b) Active tilt supplies a recovery direction outside $\ker J_v$, restoring positive ascent ($\gamma_{\mathrm{art}}>0$) at the same commanded wrench.}
    \label{fig:degeneracy}
    \vspace{-0.5em}
\end{figure}

Two qualifications: the result is \emph{first-order}---no admissible constant-wrench perturbation increases $L$ to first order, not that the state cannot move---and \emph{local}, holding at $v^\star$ under Assumption~\ref{as:regular}. What it establishes is that the recovery direction available to rotor-speed redistribution vanishes exactly where readiness is optimal, so a motor-only allocator meeting the floor there must concede wrench tracking to restore authority.
 
The mechanism is concrete (Figure \ref{fig:degeneracy}a): on a symmetric octorotor in hover, $n=8$, $m=6$ leave a genuine two-dimensional null space, and since all rotors share an operating point, \eqref{eq:trace} gives every rotor $\sigma_i=m/n$ and $\nabla_v L\propto\mathbf{1}$. Accelerating one rotor group while holding the wrench forces compensating deceleration elsewhere, and since every rotor has identical thrust sensitivity and drag headroom, the readiness gained is cancelled exactly by the readiness lost---the exchange is zero-sum, which is \eqref{eq:zerosum} read backwards.



\subsection{Restoring the ascent through articulation}

Admitting the servo tilt rates $\mu=\dot\alpha\in\mathbb{R}^n$ as a second resource enlarges \eqref{eq:gamma} to
\begin{align}
    \gamma_{\mathrm{art}} = \max_{\xi,\mu}\ &\nabla_v L^{\!\top}\xi
    + \nabla_\alpha L^{\!\top}\mu \nonumber\\
    \text{s.t. } & J_v\xi + J_\alpha\mu = 0, \quad
    |\xi|\le\bar a, \quad |\mu|\le\bar u ,
    \label{eq:gart}
\end{align}
where the equality constraint holds the commanded wrench---tilt perturbs it, motors absorb the perturbation (Figure \ref{fig:degeneracy}b). This \emph{actuator-envelope} quantity bounds what the actuators could do at the current state; Section~\ref{sec:upcqp} decides what they actually do.
 
\begin{proposition}[Certified articulated ascent]\label{prop:ascent}
Let $J_v^{+}$ be the Moore--Penrose pseudoinverse; define the \emph{reduced tilt gradient} and the induced motor-correction direction
\begin{align}
    \tilde g = \nabla_\alpha L - J_\alpha^{\!\top}\big(J_v^{+}\big)^{\!\top}\nabla_v L,
    \qquad
    z = J_v^{+}J_\alpha\operatorname{sign}(\tilde g).
    \label{eq:gtilde}
\end{align}
Then $\gamma_{\mathrm{art}}\ \ge\ \beta\,\|\tilde g\|_1\,\bar u$, with the dimensionless feasibility factor $\beta=\min\{1,\ \min_i \bar a_i/(\bar u|z_i|)\}$.
\end{proposition}
 
\begin{proof}
Take the tilt rate $\mu^\star=\bar u\operatorname{sign}(\tilde g)$, which saturates the servo box, and the minimum-norm motor correction $\xi^\star=-J_v^{+}J_\alpha\mu^\star =-\bar u z$ holding the wrench, so the pair satisfies the equality constraint of \eqref{eq:gart} by construction. Substituting into the objective and using $\operatorname{sign}(\tilde g)^{\!\top}\tilde g=\|\tilde g\|_1$ gives $\nabla_v L^{\!\top}\xi^\star+\nabla_\alpha L^{\!\top}\mu^\star=\|\tilde g\|_1\bar u$. The pair need not be admissible, since $\xi^\star$ may exceed the motor box; but scaling both components by $\beta\in(0,1]$ preserves the equality constraint, scales the objective linearly, and by the definition of $\beta$ yields $|\beta\xi^\star_i|=\beta\bar u|z_i|\le \bar a_i$ for every $i$. The scaled pair is thus feasible and attains $\beta\|\tilde g\|_1\bar u$, which lower-bounds the maximum.
\end{proof}
 
The factor $\beta$ is the price of holding the wrench: tilting perturbs it, motors must undo the perturbation, and $\beta<1$ when they lack headroom to do so at full servo rate.
 
At $v^\star$ the construction simplifies: since $\nabla_v L$ vanishes there, the second term of \eqref{eq:gtilde} drops and $\tilde g$ becomes purely geometric, $\tilde g_i = 2A_{\bullet i}^{\!\top}(AA^{\!\top})^{-1}A'_{\bullet i}$, independent of spin, alternating in sign around a symmetric airframe---opposite-sign tilt on adjacent arms inflates readiness volume while induced wrench perturbations cancel in pairs. A gradient proportional to $\mathbf{1}$ cannot do this in the rotor-speed null space, where the same pairing cancels the readiness gain as well; articulation escapes the trade because it changes the geometry of $A(\alpha)$ rather than redistributing effort within a fixed one, avoiding the drag bottleneck entirely.
 
\begin{remark}[Bidirectionality]\label{rem:bidirectional}
Articulation is not uniformly beneficial: an arbitrary tilt perturbation may lower $L$ as easily as raise it. Proposition~\ref{prop:ascent} is an existence statement: a favorable direction exists where rotor-speed redistribution has none.
\end{remark}



\subsection{Why the certificate excludes servo authority}
 
Servo motion produces wrench rate, so a metric crediting it exists: $D_{\mathrm{art}} = D + \sum_i \rho_i A'_{\bullet i}A'^{\top}_{\bullet i}$ with $\rho_i=\bar u^2\phi_i(v)^2$. We deliberately do not certify against it.
 
\begin{proposition}[Conservatism and saturation indication]\label{prop:optionA}
\emph{(i)} $D \preceq D_{\mathrm{art}}$ in the L\"owner order, hence $L\le L_{\mathrm{art}}$ pointwise and $\{L\ge\ell_{\min}\}\subseteq\{L_{\mathrm{art}}\ge\ell_{\min}\}$: certifying $h\ge0$ certifies the articulated authority as well. 
\emph{(ii)} If $\operatorname{rank}\big[A'_{\bullet 1}\phi_1,\dots,A'_{\bullet n}\phi_n\big]=m$ at saturation, then $L_{\mathrm{art}}$ remains finite as $v\to v^{\mathrm{sat}}$, whereas $L\to-\infty$.
\end{proposition}
 
\begin{proof}
(i) The difference $D_{\mathrm{art}}-D=\sum_i\rho_i A'_{\bullet i}A'^{\top}_{\bullet i}$ is a sum of rank-one positive semidefinite terms, hence $\succeq0$. Monotonicity of $\ln\det$ on the positive-definite cone gives $L\le L_{\mathrm{art}}$, and the set inclusion is immediate.
(ii) As $v\to v^{\mathrm{sat}}$ we have $\bar a\to0$ by \eqref{eq:abar}, so $\Psi\to0$, $D\to0$ and $L\to-\infty$. The servo weights behave oppositely: $\rho_i\to\bar u^2\phi_i(v^{\mathrm{sat}})^2>0$, since $\phi_i$ is evaluated at the \emph{largest} admissible spin. Under the rank hypothesis $\det D_{\mathrm{art}}$ tends to a finite positive limit.
\end{proof}
 
Part~(ii) forces the choice: under a gust the response rests on motors (high electrical bandwith), while servos are mechanically slow. A barrier built on $D_{\mathrm{art}}$ would nevertheless read those slow kinematic limits as authority---the \emph{ghost capacity} fallacy---reporting a vehicle with no remaining torque margin as safe. Restricting the metric to motor capacity avoids this and, by part~(i), costs only conservatism.



\subsection{The admissible tilt set}
 
Tilt-dependent allocation matrices lose rank at flat configurations \cite{Allenspach2020}, usually handled by a dedicated constraint; here the floor supplies it.
 
\begin{proposition}[Singularity exclusion]\label{prop:tiltbox}
Let $L^{\max}(\alpha)=\max_v L(v,\alpha)$ and $\mathcal{A}(\ell_{\min})=\{\alpha : L^{\max}(\alpha)>\ell_{\min}\}$. Then $\mathcal{A}(\ell_{\min})$ contains no configuration with $\operatorname{rank}A(\alpha)<m$, and on $\mathcal{A}(\ell_{\min})$ the certified set $\{h\ge0\}$ has nonempty interior in $v$.
\end{proposition}
 
\begin{proof}
If $\operatorname{rank}A(\alpha)<m$ then $D(v,\alpha)=4A\Psi A^{\!\top}$ is rank-deficient for every $v$ by \eqref{eq:D}, so $\det D=0$, $L^{\max}(\alpha)=-\infty<\ell_{\min}$, and $\alpha\notin\mathcal{A}(\ell_{\min})$. For the second claim, some $v$ attains $L(v,\alpha)>\ell_{\min}$; continuity of $L$ in $v$ gives a neighborhood with $h>0$.
\end{proof}
 
The floor defending authority thus also fences the singularity, needing no separate avoidance logic. The two effects pull oppositely on cant angle: the available ascent $\|\tilde g\|_1$ is largest at shallow cant, while $\mathcal{A}(\ell_{\min})$ excludes the shallowest configurations precisely because they approach rank deficiency. 



\section{Enforcement Through Physical Commands}
\label{sec:upcqp}

What remains is enforcing the floor using hardware commands---motor torques and servo setpoints---rather than the idealized actuator rates of \eqref{eq:gart}. 



\subsection{Barrier dynamics and relative degree}

Write $x=(v,\alpha)$, differentiating $h$ along \eqref{eq:motor}--\eqref{eq:servo} gives
\begin{align}
    \dot h &= \sum_{i} \frac{\partial L}{\partial v_i}\,
              \frac{\tau_i - b_i v_i|v_i|}{J_{m,i}}
            + \sum_{i} \frac{\partial L}{\partial \alpha_i}\,
              \frac{\alpha_{c,i}-\alpha_i}{\tau_s} \nonumber\\
           &= a_{\mathrm{CBF}}(x)^{\!\top} u \;-\; \delta_h(x), 
    \label{eq:hdot_affine}
\end{align}
which is affine in the physical command $u=[\tau^{\!\top},\alpha_c^{\!\top}]^{\!\top}$. Here
\begin{equation}
\resizebox{0.89\columnwidth}{!}{$
\begin{aligned}
    a_{\mathrm{CBF}}
      = \begin{bmatrix}
          D_m\nabla_v L \\[1pt]
          D_s\nabla_\alpha L
         \end{bmatrix},
    \quad
    \delta_h
      = \nabla_v L^{\!\top}D_m\,b\odot\phi(v)
       + \nabla_\alpha L^{\!\top}D_s\,\alpha,
\end{aligned}
$}
\label{eq:acbf}
\end{equation}
with $D_m=\diag(J_{m,i}^{-1})$ and $D_s=\tau_s^{-1}I$. The scalar $\delta_h$ is the \emph{barrier drift}, 
distinct from the wrench-rate drift $d$ of \eqref{eq:Md}. Both gradients come from Lemma~\ref{lem:gradients} in closed form, so \eqref{eq:acbf} costs one factorization of $D$ per step.
 
\begin{assumption}\label{as:reg}
$a_{\mathrm{CBF}}(x)\neq0$, equivalently
$\big(\nabla_v L,\nabla_\alpha L\big)\neq(0,0)$, on the operating set.
\end{assumption}
 
Under Assumption~\ref{as:reg} the physical commands appear in $h$'s first derivative with a nonvanishing coefficient, so the barrier has relative degree one and high-order constructions \cite{Breeden2021}---with the added conservatism they impose---are not required.

Assumption~\ref{as:reg} cannot be strengthened to a sign or magnitude condition, 
since the coefficient multiplying $\tau_i$ is $(\partial L/\partial v_i)/J_{m,i}$, zero wherever $\nabla_v L$ is---exactly the configuration of Proposition~\ref{prop:degeneracy}, where the motor block of $a_{\mathrm{CBF}}$ vanishes and the servo block carries the entire row. We therefore monitor Assumption~\ref{as:reg} online rather than assert it globally.



\subsection{The unified physical-command quadratic program}

Let $w_{\mathrm{des}}$ be the wrench requested by the SE(3) outer loop and $w=A(\alpha)\phi(v)$ the wrench currently produced; the target wrench rate
\begin{align}
    \dot w_{\mathrm{tar}} = K_w\big(w_{\mathrm{des}} - w\big), \qquad K_w > 0,
    \label{eq:wtar}
\end{align}
is a proportional feedback evaluated at every control instant. 

\begin{definition}[\ac{UPC-QP}]\label{def:upcqp}
With $M,d$ from \eqref{eq:affine}--\eqref{eq:Md}, weights $W\succ0$ and $R\succ0$, trim $u_{\mathrm{ref}}=[\,b\odot\phi(v)^{\!\top},\ \alpha^{\!\top}]^{\!\top}$ and barrier gain $\chi>0$,
\begin{align}
    u^\star = \arg\min_{u}\ &
    \tfrac12\big\|Mu + d - \dot w_{\mathrm{tar}}\big\|_W^2
    + \tfrac12\big\|u-u_{\mathrm{ref}}\big\|_R^2
    \label{eq:obj}\\
    \mathrm{s.t.}\quad
    & a_{\mathrm{CBF}}^{\!\top}u \;\ge\; -\chi\,h + \delta_h,
    \label{eq:cbfrow}\\
    & |\tau|\le\bar\tau,
    \label{eq:taubox}\\
    & \alpha_c \in [\alpha-\tau_s\bar u,\ \alpha+\tau_s\bar u]
      \cap[\underline\alpha,\bar\alpha].
    \label{eq:alphabox}
\end{align}
\end{definition}

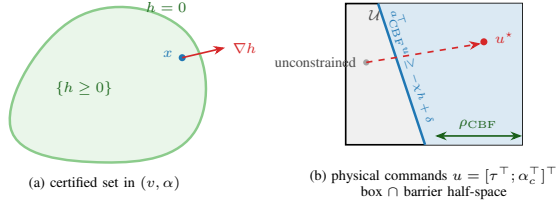
\begin{figure}[tb]
    \centering
    \resizebox{0.85\columnwidth}{!}{%
        \input{figures/tikz/fig_cbfgeom.tex}}
    \vspace{-0.5em}    
    \caption{\small (a)~The certified set $\{h\ge0\}$ in the actuator state $x=(v,\alpha)$, with outward gradient $\nabla h$ near the boundary. (b)~The same requirement in physical-command space: the barrier row \eqref{eq:cbfrow} cuts a half-space from the actuator box $\mathcal{U}$, and $u^\star$ projects the unconstrained minimizer of \eqref{eq:obj} onto it. The margin $\rho_{\mathrm{CBF}}$ of \eqref{eq:rho} is positive exactly when the row is enforceable.}
    \label{fig:cbfgeom}
    \vspace{-0.5em}
\end{figure}

Two features: \eqref{eq:cbfrow} is $\dot h\ge-\chi h$ rewritten through \eqref{eq:hdot_affine}, a single barrier inequality (Figure~\ref{fig:cbfgeom}b); and \eqref{eq:alphabox} encodes the servo rate limit \emph{exactly} rather than as a penalty, so the returned command is always executable.
 
\begin{proposition}[Well-posedness and redundancy resolution]\label{prop:well}
The objective \eqref{eq:obj} is strictly convex, the box is nonempty compact, and whenever \eqref{eq:cbfrow} is compatible with that box the minimizer $u^\star$ is unique.
\end{proposition}
 
\begin{proof}
The Hessian of \eqref{eq:obj} is $H=M^{\!\top}WM+R$. Since $M\in\mathbb{R}^{m\times 2n}$ with $2n>m$, the tracking term $M^{\!\top}WM$ is singular: it is positive semidefinite with a kernel of dimension at least $2n-m$. Adding $R\succ0$ gives $z^{\!\top}Hz\ge z^{\!\top}Rz>0$ for all $z\neq0$, so $H\succ0$ and the objective is strictly convex. The box is nonempty because the current $\alpha$ satisfies \eqref{eq:alphabox} by construction and $\tau=0$ satisfies \eqref{eq:taubox}, and it is compact as a finite product of closed bounded intervals. Intersecting with the half-space \eqref{eq:cbfrow} preserves convexity and compactness, and a strictly convex function attains a unique minimum on a nonempty compact convex set.
\end{proof}
 
This also resolves the redundancy without an explicit null-space parameterization: on the $(2n-m)$-dimensional kernel of $M^{\!\top}WM$ the tracking term is blind, and the $R$-weighted term selects the command closest to trim $u_{\mathrm{ref}}$. When $w_{\mathrm{des}}$ is not instantaneously reachable, the solver degrades tracking optimally in the metric $W$ while \eqref{eq:cbfrow}--\eqref{eq:alphabox} continue to hold exactly---the certificate is never traded against the task.



\subsection{Feasibility margin and conditional invariance}

The guarantee below is conditional on \eqref{eq:cbfrow} being satisfiable within the box, a condition that is checkable in closed form.
 
\begin{definition}[Feasibility margin]\label{def:rho}
With $\mathcal{U}$ the box \eqref{eq:taubox}--\eqref{eq:alphabox},
\begin{align}
    \rho_{\mathrm{CBF}}(x)
    = \max_{u\in\mathcal{U}} a_{\mathrm{CBF}}^{\!\top}u + \chi h - \delta_h .
    \label{eq:rho}
\end{align}
\end{definition}
 
Because $a_{\mathrm{CBF}}^{\!\top}u$ is linear and $\mathcal{U}$ is a Cartesian product of scalar intervals, the maximizer is the vertex aligned with $\operatorname{sign}(a_{\mathrm{CBF}})$ and \eqref{eq:rho} evaluates in $\mathcal{O}(n)$ operations, 
available \emph{only} because the admissible command set is a box; coupled actuator constraints would need a linear program instead.
 
Geometrically, $\rho_{\mathrm{CBF}}$ is the signed distance by which the best available command clears the barrier row (Figure~\ref{fig:cbfgeom}b), evaluated at every instant; a positive value certifies the row is enforceable at that state. On $\rho_{\mathrm{CBF}}\le0$ the controller terminates and reports---\eqref{eq:cbfrow} is never relaxed into a soft penalty, since a silently dropped constraint would void the certificate precisely when it binds.

\begin{theorem}[Conditional forward invariance]\label{thm:invariance}
Let $h$ be continuously differentiable on an open set containing the closed-loop trajectory, let Assumption~\ref{as:reg} hold along it, and suppose $\rho_{\mathrm{CBF}}(x(t))\ge0$ on $[0,T]$ with a locally Lipschitz solution map $x\mapsto u^\star(x)$. If $h(x(0))\ge0$ then $h(x(t))\ge0$ for all $t\in[0,T]$.
\end{theorem}
 
\begin{proof}
Feasibility means the applied $u^\star$ satisfies \eqref{eq:cbfrow}, which by \eqref{eq:hdot_affine} reads $\dot h\ge-\chi h$ pointwise along the closed-loop vector field. Local Lipschitz continuity of $u^\star$ gives existence and uniqueness of the closed-loop solution on $[0,T]$. Comparing $h$ with $y(t)=h(x(0))e^{-\chi t}$ (solving $\dot y=-\chi y$) via the comparison lemma yields $h(x(t))\ge h(x(0))e^{-\chi t} \ge0$ iff $h(x(0))\ge0$.
\end{proof}
 
Theorem~\ref{thm:invariance} states the trajectory never leaves the certified set (Figure~\ref{fig:cbfgeom}a). Its hypotheses are monitored, not established once. Feasibility is certified online through $\rho_{\mathrm{CBF}}$ and Assumption~\ref{as:reg} is checked each step; we do not establish the Lipschitz property of the solution map, and strict convexity gives uniqueness of $u^\star$ but not continuity in $x$. A stronger, unconditional statement would require proving feasibility over the whole operating envelope, which we have not done.



\section{Simulation Study}
\label{sec:sim}

Validation uses a reconstructed actively articulated octorotor ($n=8$, $m=6$). Structural parameters, rotor placement, and the wrench map matrix derive from the Omnidirectional Octorotor architecture \cite{Hamandi2025prototype}, augmented here with servo-actuated tilt; the physical parameters (Table~\ref{tab:params}) combine literature measurements with synthesized capacity bounds. The torque limit $\bar{\tau}$ is set so $v^{\mathrm{sat}}=\sqrt{\bar{\tau}/b}$ coincides with the manufacturer's maximum rotor speed, coupling the capacity model to empirical actuator envelopes. The operating point $v^\star=v^{\mathrm{sat}}/\sqrt{3}=\SI{577.4}{\radian\per\second}$ follows from \eqref{eq:psi}, and the readiness floor is $\ell_{\min}=L(v^\star,\alpha_0)-2=26.098$ nats. The lumped drag coefficient $b$ is identified with the static propeller torque coefficient $c_\tau$; the resulting non-dimensional coefficients ($C_T=0.101$, $C_Q=0.0071$, $C_P=0.044$ at $\rho=\SI{1.225}{\kilo\gram\per\metre\cubed}$, $D=\SI{0.229}{\metre}$) fall within the UIUC Propeller Database's \cite{UIUCPropDB} thin-electric spectrum.

\begin{table}[t]
    \caption{\small Platform parameters.}
    \label{tab:params}
    \vspace{-0.5em}
    \renewcommand{\arraystretch}{0.7}
    \centering\footnotesize\setlength{\tabcolsep}{3pt}
    \begin{adjustbox}{max width=0.98\columnwidth}
    \begin{tabular}{@{}lll@{\hspace{3ex}}lll@{}}
    \toprule
    Symbol & Value & Unit & Symbol & Value & Unit \\
    \midrule
    $n$, $m$                      & $8$, $6$                      & ---                               & $J_m$              & $5\times10^{-5}$ & \si{\kilo\gram\metre\squared} \\
    arm length                    & $0.246$                       & \si{\metre}                       & $\tau_s$           & $0.05$           & \si{\second} \\
    mass                          & $2.0$                         & \si{\kilo\gram}                   & $b$                & $=c_\tau$        & --- \\
    inertia                       & $\diag(0.0217,0.0217,0.04)$ & \si{\kilo\gram\metre\squared}     & $\bar\tau$         & $0.137$          & \si{\newton\metre} \\
    $c_f$                         & $8.59\times10^{-6}$           & \si{\newton\second\squared}       & $v^{\mathrm{sat}}$ & $1000$           & \si{\radian\per\second} \\
    $c_\tau$                      & $1.37\times10^{-7}$           & \si{\newton\metre\second\squared} & $v^\star$          & $577.4$          & \si{\radian\per\second} \\
    $\bar u$                      & $180$--$462$                  & \si{\degree\per\second}           & $\ell_{\min}$      & $26.098$         & nats \\
    $\underline\alpha,\bar\alpha$ & $\pm30$                       & \si{\degree}                      &                    &                  & \\
    \bottomrule
    \end{tabular}
    \end{adjustbox}
\end{table}

To isolate the contributions of structural articulation and the proposed \ac{CBF}, four allocation schemes are evaluated under an identical SE(3) tracking loop:
\begin{enumerate}\itemsep1pt \parskip0pt \topsep2pt
    \item \textbf{Static Pseudo-Inverse:} standard minimum-norm baseline \cite{Johansen2013} $\phi_{\mathrm{des}}=A^{+}w_{\mathrm{des}}$, fixed tilt $\alpha_{0,i}=\pm\SI{15}{\degree}$, ex-post actuator clipping.
    \item \textbf{Fixed-Tilt \ac{DAAM}:} Definition~\ref{def:upcqp} with static tilt and the barrier constraint \eqref{eq:cbfrow} omitted.
    \item \textbf{Uncertified Articulated \ac{DAAM}:} Definition~\ref{def:upcqp} with active tilt, but without the barrier constraint \eqref{eq:cbfrow}.
    \item \textbf{Robust CBF Filter:} Definition~\ref{def:upcqp}, incorporating both active articulation and the barrier constraint \eqref{eq:cbfrow}.
\end{enumerate}

System dynamics are integrated explicitly at $\SI{5}{\milli\second}$ over a \SI{10}{\second} horizon across five environments: a benign hover, a $\SI{0.5}{\metre}$ position step, an aggressive lateral maneuver, and that maneuver under mild (\SI{3}{\newton}) and strong (\SI{8}{\newton}) raised-cosine lateral gusts. The complete pipeline---plant models, allocator parameters, seeds, and data-generation scripts---is at \coderepo; the allocator is implemented in Python, with \eqref{eq:obj} solved as a bounded least-squares program.

The gusts enter as external additive forces on the airframe; rotor coefficients $c_f, c_\tau$ and the allocation map $A(\alpha)$ remain at static-hover values throughout, isolating the allocation problem from the fluid dynamics of high-advance-ratio crosswinds. Modeling the inflow-dependent degradation a true \SI{8}{\newton} gust would induce is left to future work.

\begin{table}[t]
    \caption{\small Closed-loop metrics ($\bar u=\SI{276}{\degree\per\second}$). $h_{\min}<0$ denotes strict violation of the authority floor.} 
    \label{tab:e4}
    \vspace{-0.5em}
    \renewcommand{\arraystretch}{0.7}
    \centering\footnotesize\setlength{\tabcolsep}{3pt}
    \begin{adjustbox}{max width=0.98\columnwidth}
    \begin{tabular}{@{}llccccc@{}}
    \toprule
    Task & Allocator & RMS$_p$ [m] & $h_{\min}$ & sat.\ [\%] & tilt [$^\circ$] & RMS$_w$ \\
    \midrule
    Hover       & All & 0.0000 & $+1.940$ & 0.00 & 0.00 & 0.000 \\
    \midrule
    Aggressive  & Static P-Inv & 1.0263 & $+1.681$ & 0.00 & 0.00 & 0.011 \\
                & Fixed-Tilt DAAM & 1.0250 & $+1.567$ & 0.00 & 0.00 & 0.075 \\
                & Uncertified DAAM & 1.0250 & $+1.625$ & 0.00 & 0.08 & 0.075 \\
                & Robust CBF Filter & 1.0250 & $+1.625$ & 0.00 & 0.08 & 0.075 \\
    \midrule
    Mild gust   & Static P-Inv & 1.0738 & $\mathbf{-0.440}$ & 0.00 & 0.00 & 0.016 \\
                & Fixed-Tilt DAAM & 1.0737 & $\mathbf{-0.749}$ & 0.00 & 0.00 & 0.101 \\
                & Uncertified DAAM & 1.0737 & $\mathbf{-0.156}$ & 0.00 & 0.23 & 0.101 \\
                & Robust CBF Filter & 1.0737 & $\mathbf{+0.251}$ & 0.00 & 1.06 & 0.101 \\
    \midrule
    Strong gust & Static P-Inv & \multicolumn{5}{c}{\emph{diverged at $t=\SI{6.005}{\second}$}} \\
                & Fixed-Tilt DAAM & 1.5461 & $\mathbf{-63.25}$ & 20.99 & 0.00  & 1.320 \\
                & Uncertified DAAM & 1.4444 & $\mathbf{-15.30}$ & 0.18  & 11.80 & 0.214 \\
                & Robust CBF Filter & 1.4444 & $\mathbf{+0.012}$ & 0.21  & 14.98 & 0.212 \\
    \bottomrule
    \end{tabular}
    \end{adjustbox}
\end{table}

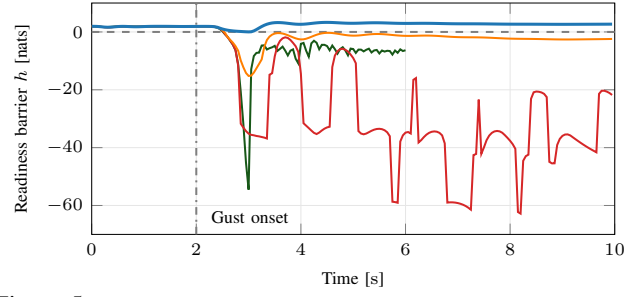
\begin{figure}[tb]
    \centering
    \resizebox{0.95\columnwidth}{!}{\input{figures/tikz/fig_readiness.tex}}
    \vspace{-0.75em}
    \caption{\small Readiness trajectories under the strong gust: Static Pseudo-Inverse (green), Fixed-Tilt \ac{DAAM} (red), Uncertified Articulated \ac{DAAM} (orange), Robust \ac{CBF} Filter (blue). Readiness floor $h=0$ (dashed) and gust onset at $t=\SI{2}{\second}$ (dash-dotted), both gray.}
    \label{fig:readiness}
    \vspace{-0.75em}
\end{figure}

\emph{Authority preservation and closed-loop tracking.} When authority is abundant (hover, aggressive), Table~\ref{tab:e4} shows all allocators tracking nearly identically and the barrier strictly inactive  ($0$ of $2000$ instants, $h \ge 1.567$): the \ac{CBF} costs nothing outside the critical boundary. The Static Pseudo-Inverse is most accurate in wrench tracking ($0.011$ vs. $0.075$ \ac{RMS}) since it inverts the static map directly rather than through the rate-based \ac{DAAM} formulation.

Under the mild gust, tracking stays uniform across allocators (isolating readiness as the only quantity that moves) while every uncertified allocator breaches the floor; only the Robust \ac{CBF} Filter holds it, at \SI{1.06}{\degree} of tilt. Under the strong gust the Static Pseudo-Inverse diverges at $t=\SI{6.005}{\second}$, having reached $49.2\%$ saturation and $h_{\min}=-54.53$. Active articulation alone (Uncertified \ac{DAAM}) recovers most of the actuator-usage benefit relative to fixed geometry---saturation and wrench error both fall by roughly an order of magnitude---yet still violates the floor by $15.30$ nats: articulation is necessary but not sufficient. Re-solving with and without \eqref{eq:cbfrow} at every instant shows the barrier row active on only $5.8\%$ of steps (median command displacement  $11.4\%$): it converts the Uncertified \ac{DAAM}'s recovered capacity into a certified floor, bounding the state at the envelope's edge ($h_{\min}=+0.012$) without added saturation (Figure~\ref{fig:readiness}).

The commands issued are physically executable (Figure~\ref{fig:inputs}): rotor speeds stay within $[0.20,0.89]\,v^{\mathrm{sat}}$, reaching neither saturation nor a zero-crossing---the two loci at which the certificate collapses---and tilt stays inside $\pm30^\circ$, touching the limit only transiently. The vehicle rejects the disturbance through articulation, as Section~\ref{sec:art} predicts, rather than forcing motors past capacity.
 
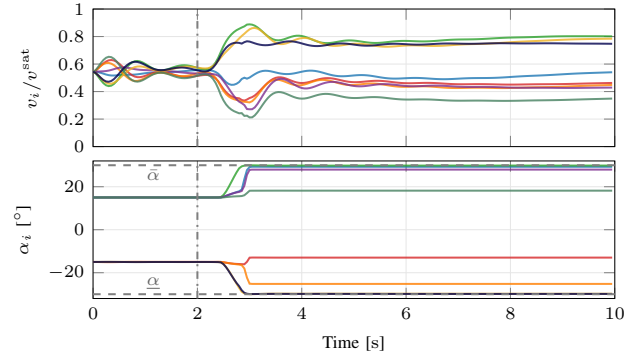
\begin{figure}[tb]
    \centering
    \resizebox{0.95\columnwidth}{!}{\input{figures/tikz/fig_inputs.tex}}
    \vspace{-0.75em}
    \caption{\small Physical actuator commands under the strong gust, Robust \ac{CBF} Filter (disturbance onset $t=\SI{2}{\second}$, dashed gray). Top: rotor speeds, avoiding saturation and zero-crossings. Bottom: tilt angles, using the mechanical range to absorb the commanded wrench.}
    \label{fig:inputs}
    \vspace{-0.75em}
\end{figure}

%


\emph{Robustness to parametric uncertainty.} We first perturb the plant parameters while leaving the model class intact. On the strong-gust trajectory, mass ($\pm10\%$), torque limit ($\pm15\%$), and thrust coefficient ($\pm10\%$) are randomized and the servo-rate limit is drawn from $\{180,276,318,462\}\,\si{\degree\per\second}$, over eight trials with shared initial states and seeds across allocators. The Robust \ac{CBF} Filter holds $h_{\min}>0$ in all eight (median $+0.018$, worst $+0.004$) against $0/8$ for the Fixed-Tilt \ac{DAAM} (median $-58.24$), at lower median \ac{RMS} position error (\SI{1.565}{} vs.\ \SI{1.689}{\metre}). Since $\rho_{\mathrm{CBF}}$ stays positive throughout (minimum $81.85$), Theorem~\ref{thm:invariance}'s hypotheses hold, and the $+0.004$ worst case shows the filter tracking the true capacity boundary rather than resting on slack.


\emph{Robustness to unmodeled aerodynamics.} We next introduce a \emph{structural} mismatch. The allocator is unchanged---barrier, gradients, and allocation follow the static thrust map $A(\alpha)\phi(v)$ of \eqref{eq:wrench}---but the integrator applies a higher-fidelity wrench it never sees, computed from the air-relative velocity $v_{\mathrm{air}} = v_{\mathrm{body}} - R^{\!\top}v_{\mathrm{wind}}$: advance-ratio thrust loss, an in-plane rotor $H$-force with induced torque, and quadratic fuselage drag. The Robust \ac{CBF} Filter still holds the floor ($h_{\min}=+0.012$), articulating to absorb the discrepancy as bounded wrench-tracking error via the $W$-weighted cost of \eqref{eq:obj} while the barrier stays hard. The uncertified allocators fail more severely than before: Fixed-Tilt \ac{DAAM} collapses to $h_{\min}=-63.25$ at $21.0\%$ saturation, the Static Pseudo-Inverse diverges at $t=\SI{6.01}{\second}$. Surviving a plant it does not model, the certificate enforces \emph{structural}, not model-matched, safety.



\section{Critical Discussion and Future Work}\label{sec:disc}

The readiness level is a log-volume, which can be large while the reachable set is thin in one direction, so $h\ge0$ does not imply authority is available in whatever direction a disturbance demands. To test this we defined, for analysis only, the worst-direction authority
\begin{align}
    \eta(v,\alpha) = \min_{\|y\|=1}\ \max\big\{y^{\!\top}\dot w \;:\;
    \dot w \in J_v(v,\alpha)\,\mathcal{B}_{\bar a}\big\},
    \label{eq:eta}
\end{align}
with $\mathcal{B}_{\bar a}=\{\xi: |\xi|\le\bar a\}$ the motor-rate box: the smallest, over unit wrench-rate directions, of the largest response that box can deliver in that direction, one linear program per sampled direction. Over $40$ states sampled along the gust trajectories with $60$ directions each, states with $h<0$ have median $\eta=2.74$ against $31.25$ for $h\ge0$---an order of magnitude, precisely the collapse the floor exists to prevent. Above the floor, however, correlation between $h$ and $\eta$ is $+0.04$, and between $h$ and the median-direction response $-0.34$.  Consequently, $\ln\det D$ is used strictly as a threshold, avoiding a directional linear program in the control loop. Restricting the certificate to motor authority also introduces measurable conservatism: at the operating point $v^\star$, the discarded marging $L_\mathrm{art}-L$ is $4.20$, $6.38$, $7.19$, and $9.49$ nats at the four commercial servo rates of Table~\ref{tab:params} ($180$--$\SI{462}{\degree\per\second}$)---excluding servo capacity to preserve the saturation divergence (Proposition~\ref{prop:optionA}) leaves this margin uncredited, particularly for high-speed actuators---a trade-off that preserves single-inequality enforceability, leaving a unified motor-servo readiness barrier an open problem.

\emph{Theoretical boundaries and algorithmic scope.}
Theorem~\ref{thm:invariance}'s guarantee rests on three conditions: continuous online feasibility of the allocation quadratic program, monitored via $\rho_{\mathrm{CBF}}$; a locally Lipschitz solution map, though strict convexity gives only a unique $u^\star$, not continuity in the state, since the Lipschitz property can fail at active-set transitions; and the degeneracy of Proposition~\ref{prop:degeneracy}, proven for symmetric and near-symmetric configurations only, leaving its persistence under manufacturing asymmetry untested. Empirically, the filter bounds the state precisely at the envelope's edge ($h_{\min}=+0.012$), confirming the barrier formulation---not mere availability of articulation---preserves the vehicle under severe disturbance.

\emph{Methodological trade-offs and future directions.}
Integrating a rate-based \ac{CBF} into the allocation framework trades a marginal loss of wrench-tracking precision in benign flight (Section \ref{sec:sim}) for strict authority preservation under disturbances that make classical, uncertified allocators diverge; articulation supplies the physical authority, the barrier enforces its correct expenditure. These conclusions are bounded by the reconstructed plant's modeling assumptions: static drag approximations, an estimated motor inertia $J_m$, and unmodeled aeroelastic and servo-loading effects introduce parametric uncertainty in absolute readiness values, though the underlying geometric and thresholding properties remain invariant. Future work: hardware validation, broader robustness testing, and a less conservative barrier.



\bibliographystyle{IEEEtran}
\bibliography{references}


\end{document}

%% file: figures/tikz/fig_section.tex
\begin{tikzpicture}[font=\scriptsize, line width=0.5pt, ar/.style={-{Latex[length=2mm]}}]
    \def\W{8.2}   
    \def\H{2.55}  
    
    \fill[cOrange!16] (0,-0.62) rectangle (\W,0.62);
    \fill[pattern=north east lines,pattern color=cRed!30] (0,2.15) rectangle (\W,\H);
    \fill[pattern=north east lines,pattern color=cRed!30] (0,-\H) rectangle (\W,-2.15);
    
    \fill[cBlue!17] (0,0.98) rectangle (\W,1.98);
    \fill[cBlue!17] (0,-1.98) rectangle (\W,-0.98);
    \draw[cBlue!70,line width=0.6pt] (0,0.98) -- (\W,0.98);
    \draw[cBlue!70,line width=0.6pt] (0,1.98) -- (\W,1.98);
    \draw[cBlue!70,line width=0.6pt] (0,-0.98) -- (\W,-0.98);
    \draw[cBlue!70,line width=0.6pt] (0,-1.98) -- (\W,-1.98);
    
    \draw[ar] (0,-\H-0.15) -- (0,\H+0.32) node[below left,font=\scriptsize]
      {$v_k$};
    \draw[ar] (-0.15,0) -- (\W+0.25,0) node[below left,font=\scriptsize]
      {$w$};
    \draw[cGrey!70,dotted] (0,1.48) -- (\W,1.48);
    \draw[cGrey!70,dotted] (0,-1.48) -- (\W,-1.48);
    \node[anchor=east,font=\tiny,cGrey!70!black] at (-0.06,1.48) {$+v^\star$};
    \node[anchor=east,font=\tiny,cGrey!70!black] at (-0.06,-1.48) {$-v^\star$};
    \node[anchor=east,font=\tiny,cOrange!80!black] at (-0.06,0) {$0$};
    
    \draw[cGrey!75,line width=1.1pt] plot[smooth,tension=0.7]
      coordinates {(0.15,1.52)(2.0,1.60)(4.0,1.50)(6.0,1.35)(7.9,1.12)};
    \draw[cGrey!75,line width=1.1pt] plot[smooth,tension=0.7]
      coordinates {(0.15,-1.52)(2.0,-1.60)(4.0,-1.50)(6.0,-1.35)(7.9,-1.12)};
    \node[cGrey!60!black,font=\tiny,anchor=west] at (7.95,1.12) {$\Sigma_+$};
    \node[cGrey!60!black,font=\tiny,anchor=west] at (7.95,-1.12) {$\Sigma_-$};
    
    \draw[cOrange,line width=1.2pt] plot[smooth,tension=0.7]
      coordinates {(0.30,1.53)(1.6,1.60)(3.1,1.53)};
    \draw[cOrange,line width=1.2pt] plot[smooth,tension=0.7]
      coordinates {(3.1,-1.53)(5.0,-1.44)(7.6,-1.14)};
    \draw[cOrange,line width=1.2pt,dashed,{Latex[length=1.6mm]}-{Latex[length=1.6mm]}]
      (3.1,1.49) -- (3.1,-1.49);
    \fill[cOrange] (3.1,1.53) circle (1.5pt); \fill[cOrange] (3.1,-1.53) circle (1.5pt);
    
    \draw[cGreen!65!black,line width=1.3pt] plot[smooth,tension=0.7]
      coordinates {(0.30,1.53)(1.6,1.58)(3.1,1.50)(4.6,1.40)(6.2,1.28)(7.6,1.18)};
    \fill[cGreen!65!black] (7.6,1.18) circle (1.5pt);
    
    \draw[{Latex[length=1.5mm]}-{Latex[length=1.5mm]},cGreen!45!black,line width=0.6pt] (7.6,1.05) -- (7.9,1.05);
    \draw[cGreen!45!black, densely dotted] (7.6, 1.18) -- (7.6, 1.05);
    \draw[cGrey!75, densely dotted] (7.9, 1.12) -- (7.9, 1.05);
    \node[cGreen!45!black,font=\tiny,anchor=north,align=center] at (7.75,1.02)
      {slack $\delta$};
    
    \node[cOrange!80!black,font=\tiny,anchor=west,align=left] at (0.12,0.34)
      {zero-crossing (loss of authority)};  
    \node[cRed,font=\tiny,anchor=west] at (0.12,2.36) {motor saturation};
    \node[cRed,font=\tiny,anchor=west] at (0.12,-2.36) {motor saturation};
    \node[cBlue,font=\tiny,anchor=west] at (0.12,1.80) {certified safe band};
    \node[cBlue,font=\tiny,anchor=west] at (0.12,-1.83) {certified safe band};

\end{tikzpicture}

%% file: figures/tikz/fig_geometry.tex
\begin{tikzpicture}[font=\scriptsize,
  >={Stealth[length=2mm]},
  lbl/.style={font=\scriptsize, align=center}]

    \begin{scope}[shift={(0,0)}]
      \draw[thick, fill=cGrey!12] (-0.95,-0.95) rectangle (0.95,0.95);
      \draw[->, cGrey] (-1.25,0) -- (1.35,0) node[right]{$\dot v_1$};
      \draw[->, cGrey] (0,-1.25) -- (0,1.35) node[above]{$\dot v_2$};
      \draw[<->, cBlue, thick] (0.95,-1.12) -- node[below, cBlue]{$2\bar a_i$} (-0.95,-1.12);
      \node[lbl] at (0,-1.75) {(a) $|\dot v_i|\le\bar a_i(v_i)$};
    \end{scope}
    
    \draw[->, very thick] (2.0,0) -- node[above]{$J_v$} node[below]{\eqref{eq:wrenchrate}} (3.3,0);
    
    \begin{scope}[shift={(4.9,0)}]
      \draw[thick, fill=cGrey!12]
        (-1.15,0.35) -- (-0.45,1.0) -- (0.75,0.85) -- (1.15,-0.3)
        -- (0.4,-1.0) -- (-0.8,-0.75) -- cycle;
      \draw[cBlue, very thick, fill=cBlue!14, rotate=18] (0,0) ellipse (0.86 and 0.52);
      \node[cBlue] at (0.05,0.02) {$\mathcal{E}$};
      \node[lbl] at (0,-1.75) {(b) zonotope $\supseteq$ ellipsoid $\mathcal{E}$};
    \end{scope}
    
    \draw[->, very thick] (6.6,0) -- node[above]{$\det$} (7.9,0);
    
    \begin{scope}[shift={(9.25,0)}]
      \draw[cBlue, very thick, fill=cBlue!14, rotate=18] (0,0) ellipse (1.02 and 0.66);
      \draw[<->, cRed, thick, rotate=18] (-1.02,0) -- node[above, cRed, sloped, xshift=-14pt, yshift=-4pt]{$\sqrt{\nu_1}$} (1.02,0);
      \draw[<->, cOrange, thick, rotate=18] (0,-0.66) -- node[right, cOrange, xshift=-1pt, yshift=-3pt]{$\sqrt{\nu_2}$} (0,0.66);
      \node[lbl] at (0,-1.75) {(c) $\det D=\prod_k\nu_k$};
    \end{scope}
    
    \draw[->, cOrange, very thick, dashed] (4.9,1.15) to[bend left=20]
      node[above, cOrange]{$\alpha$ reshapes $A(\alpha)$, hence $\mathcal{E}$} (9.25,1.15);
\end{tikzpicture}

%% file: figures/tikz/fig_degeneracy.tex
\begin{tikzpicture}[font=\scriptsize, >={Stealth[length=2mm]},
  lbl/.style={font=\scriptsize, align=center}]

    \begin{scope}
      \draw[cGrey, fill=cGrey!8] (-1.7,-1.0) rectangle (1.7,1.6);
      \draw[cGreen, very thick] (-1.35,-0.62) -- (1.35,0.62)
          node[right, cGreen, pos=1.0, xshift=-5.5mm, yshift=-3.5mm]{$\ker J_v$};
      \draw[->, cRed, very thick] (0,0) -- (-0.5,1.09)
          node[above, cRed, yshift=-1mm, xshift=2mm]{$\nabla_v L$};
      \draw[rotate=50] (0.32,-0.12) -- (0.45,0.17) -- (0.16,0.30);  
      \fill[cBlue] (0,0) circle (1.6pt);
      \node[cBlue, below right] at (0,0) {$v^\star$};
      \node[lbl] at (0,-1.45) {(a) fixed geometry\\
        $\nabla_vL\perp\ker J_v \Rightarrow \gamma=0$};
    \end{scope}
    
    \begin{scope}[shift={(4.8,0)}]
      \draw[cGrey, fill=cGrey!8] (-1.9,-1.0) rectangle (1.9,1.6);
      \draw[cGreen, very thick] (-1.35,-0.62) -- (1.35,0.62)
          node[right, cGreen, xshift=-5.5mm, yshift=-3.5mm]{$\ker J_v$};
      \draw[->, cRed, very thick] (0,0) -- (-0.5,1.09) node[above left, cRed, yshift=-1mm, xshift=2mm]{$\nabla_vL$};
      \draw[->, cOrange, very thick] (0,0) -- (1.15,0.95)
          node[right, cOrange, align=left, xshift=-1mm]{$J_\alpha\dot\alpha$};
      \draw[->, cBlue, very thick, dashed] (0,0) -- (0.62,1.28)
          node[above, cBlue, yshift=-1mm]{$\gamma_{\mathrm{art}}>0$};
      \fill[cBlue] (0,0) circle (1.6pt);
      \node[lbl] at (0,-1.45) {(b) articulated\\
        tilt leaves the fibre at constant wrench};
    \end{scope}
\end{tikzpicture}

%% file: figures/tikz/fig_cbfgeom.tex
\begin{tikzpicture}[font=\scriptsize, >={Stealth[length=2mm]},
  lbl/.style={font=\scriptsize, align=center}]

    \begin{scope}
      \draw[cGreen!55, fill=cGreen!10, very thick]
        plot[smooth cycle, tension=0.75]
        coordinates {(-1.6,-0.9) (-0.5,0.95) (1.35,0.75) (1.55,-0.95) (0,-1.5)};
      \node[cGreen!65!black] at (-0.35,-0.25) {$\{h\ge0\}$};
      \node[cGreen!65!black, font=\scriptsize] at (1.05,1.15) {$h=0$};
      \fill[cBlue] (1.32,0.28) circle (1.7pt);
      \node[cBlue, left] at (1.28,0.34) {$x$};
      \draw[->, cRed, thick] (1.32,0.28) -- (2.05,0.45) node[right, cRed]{$\nabla h$};
      \node[lbl] at (0,-1.85) {(a) certified set in $(v,\alpha)$};
    \end{scope}
    
    \begin{scope}[shift={(5.6,0)}]
      \draw[thick, fill=cGrey!12] (-1.5,-1.2) rectangle (1.5,1.2);
      \node[cGrey!60!black] at (-1.0,1.0) {$\mathcal{U}$};
      \fill[cBlue!16] (-0.15,-1.2) -- (1.5,-1.2) -- (1.5,1.2) -- (-0.95,1.2) -- cycle;
      \draw[cBlue, very thick] (-0.15,-1.2) -- (-0.95,1.2);
      \node[cBlue, rotate=-72, anchor=south] at (-0.62,0.05)
        {\tiny $a_{\mathrm{CBF}}^{\!\top}u \ge -\chi h+\delta$};
      \fill[cRed] (0.85,0.55) circle (1.7pt);
      \node[cRed, right] at (0.92,0.62) {$u^\star$};
      \fill[cGrey!70] (-1.15,0.2) circle (1.5pt);
      \node[cGrey!60!black, left] at (-1.2,0.2) {unconstrained};
      \draw[->, cRed, thick, dashed] (-1.15,0.2) -- (0.8,0.52);
      \draw[<->, cGreen!60!black, thick] (1.5,-1.05) -- (0.02,-1.05);
      \node[cGreen!60!black, below] at (0.75,-0.65) {$\rho_{\mathrm{CBF}}$};
      \node[lbl] at (0,-1.85) {(b) physical commands $u=[\tau^\top;\alpha_c^\top]^\top$\\
        box $\cap$ barrier half-space};
    \end{scope}
\end{tikzpicture}

%% file: figures/tikz/fig_readiness.tex
\begin{tikzpicture}
    \begin{axis}[width=\columnwidth, height=4.8cm,
      xlabel={Time [\si{\second}]}, 
      ylabel={Readiness barrier $h$ [nats]}, 
      xmin=0, 
      xmax=10, 
      ymin=-70, 
      ymax=10,
      label style={font=\scriptsize}, tick label style={font=\scriptsize},
      legend style={font=\scriptsize, at={(1.2,1.00)}, anchor=south east, draw=none,
                    fill=white, fill opacity=0.85, text opacity=1, nodes={scale=0.92}},
                    legend cell align=left, grid=major, grid style={gray!20}, legend columns=-1]
      
    \addplot[cGrey, dashed, thick, forget plot] coordinates {(0,0) (10,0)};
    \addplot[cGreen!55!black, thick] table[x index=0, y index=1] {data/cl_h_E45_P.dat};
    \addplot[cRed, thick]    table[x index=0, y index=1] {data/cl_h_E45_A.dat};
    \addplot[cOrange, thick] table[x index=0, y index=1] {data/cl_h_E45_B.dat};
    \addplot[cBlue, very thick] table[x index=0, y index=1] {data/cl_h_E45_C.dat};

    \draw[cGrey, thick, dashdotted] (axis cs:2,-70) -- (axis cs:2,10);
    \node[font=\scriptsize, anchor=west] at (axis cs:2.12,-64) {Gust onset};
    \end{axis}
\end{tikzpicture}

%% file: figures/tikz/fig_inputs.tex
\begin{tikzpicture}
    \begin{axis}[name=top, width=\columnwidth, height=3.5cm,
      ylabel={$v_i/v^{\mathrm{sat}}$}, 
      xmin=0, xmax=10, ymin=0, ymax=1,
      xticklabels={}, 
      label style={font=\scriptsize}, tick label style={font=\scriptsize},
      grid=major, grid style={gray!20},
      cycle list={
        {cBlue, thick, no marks, opacity=0.8},
        {cRed, thick, no marks, opacity=0.8},
        {cGreen, thick, no marks, opacity=0.8},
        {cOrange, thick, no marks, opacity=0.8},
        {myPurple, thick, no marks, opacity=0.8},
        {myYellow, thick, no marks, opacity=0.8},
        {darkgreen, thick, no marks, opacity=0.8},
        {mydarkblue, thick, no marks, opacity=0.8}
      }]
    
    \pgfplotsinvokeforeach{1,...,8}{
      \addplot+ table[x index=0, y index=#1] {data/inputs_E45_C.dat};
    }
    
    \draw[cGrey, thick, dashdotted] (axis cs:2,0) -- (axis cs:2,1);
    \end{axis}
    
    \begin{axis}[at={(top.south west)}, anchor=north west, yshift=-2mm,
      width=\columnwidth, height=3.5cm,
      xlabel={Time [\si{\second}]}, 
      ylabel={$\alpha_i$ [\si{\degree}]}, 
      xmin=0, xmax=10, ymin=-32, ymax=32,
      label style={font=\scriptsize}, tick label style={font=\scriptsize},
      grid=major, grid style={gray!20},
      cycle list={
        {cBlue, thick, no marks, opacity=0.8},
        {cRed, thick, no marks, opacity=0.8},
        {cGreen, thick, no marks, opacity=0.8},
        {cOrange, thick, no marks, opacity=0.8},
        {myPurple, thick, no marks, opacity=0.8},
        {myYellow, thick, no marks, opacity=0.8},
        {darkgreen, thick, no marks, opacity=0.8},
        {mydarkblue, thick, no marks, opacity=0.8}
      }]
    
    \pgfplotsinvokeforeach{9,...,16}{
      \addplot+ table[x index=0, y index=#1] {data/inputs_E45_C.dat};
    }
    
    \draw[cGrey, thick, dashed] (axis cs:0,30) -- (axis cs:10,30);
    \draw[cGrey, thick, dashed] (axis cs:0,-30) -- (axis cs:10,-30);
    \node[font=\scriptsize, cGrey!80!black, anchor=north east] at (axis cs:1.45,32) {$\bar\alpha$};
    \node[font=\scriptsize, cGrey!80!black, anchor=south east] at (axis cs:1.45,-32) {$\underline\alpha$};
    \draw[cGrey, thick, dashdotted] (axis cs:2,-32) -- (axis cs:2,32);
    \end{axis}
\end{tikzpicture}